\documentclass[conference,letterpaper]{IEEEtran}
\IEEEoverridecommandlockouts
\usepackage[T1]{fontenc}
\usepackage[utf8]{inputenc}
\usepackage{amsmath,amssymb,amsthm, mathtools}
\usepackage{graphicx}
\usepackage{booktabs}
\usepackage{cite}
\usepackage{enumitem}
\usepackage[letterpaper,left=0.75in,right=0.75in,top=0.75in,
  textheight=672pt,headheight=0pt,headsep=0pt]{geometry}
\usepackage[bookmarks=false,hidelinks]{hyperref}

\usepackage{microtype}
\microtypesetup{protrusion=false}
\renewcommand{\baselinestretch}{0.80}
\makeatletter
\renewcommand{\section}{\@startsection{section}{1}{\z@}%
  {0.6ex plus 0.5ex minus 0.2ex}%
  {0.15ex plus 0.15ex minus 0ex}%
  {\normalfont\normalsize\centering\scshape}}
\renewcommand{\subsection}{\@startsection{subsection}{2}{\z@}%
  {0.5ex plus 0.4ex minus 0.2ex}%
  {0.1ex plus 0.1ex minus 0ex}%
  {\normalfont\normalsize\itshape}}
\makeatother
\setlist{topsep=2pt,itemsep=0pt,parsep=0pt,partopsep=0pt}
\newcommand{\compactdisplayskips}{%
  \setlength{\abovedisplayskip}{2pt plus 1.5pt minus 0.5pt}%
  \setlength{\belowdisplayskip}{2pt plus 1.5pt minus 0.5pt}%
  \setlength{\abovedisplayshortskip}{1pt plus 1pt minus 0.5pt}%
  \setlength{\belowdisplayshortskip}{1pt plus 1pt minus 0.5pt}%
}
\makeatletter
\g@addto@macro\normalsize{\compactdisplayskips}
\def\@IEEENORMtitlevspace{0.75\baselineskip}

\makeatother
\AtBeginDocument{\compactdisplayskips}

\newtheoremstyle{compactplain}
  {2pt plus 1pt minus 1pt}{2pt plus 1pt minus 1pt}
  {\itshape}{}{\bfseries}{.}{0.5em}{}
\newtheoremstyle{compactremark}
  {2pt plus 1pt minus 1pt}{2pt plus 1pt minus 1pt}
  {\normalfont}{}{\itshape}{.}{0.5em}{}
\newtheoremstyle{appendixplain}
  {0.5pt}{0.5pt}
  {\itshape}{}{\bfseries}{.}{0.5em}{}
\theoremstyle{compactplain}
\newtheorem{theorem}{Theorem}
\newtheorem{proposition}[theorem]{Proposition}
\newtheorem{lemma}[theorem]{Lemma}

\theoremstyle{compactremark}

\makeatletter
\renewenvironment{proof}[1][\proofname]{\par
  \pushQED{\qed}%
  \normalfont \topsep2\p@\@plus1\p@\@minus1\p@\relax
  \trivlist
  \item[\hskip\labelsep\itshape #1\@addpunct{.}]\ignorespaces
}{%
  \popQED\endtrivlist\@endpefalse
}
\makeatother

\DeclareRobustCommand{\submissionthanksbegin}{\begingroup
  \def\baselinestretch{1}\fontsize{8}{9}\selectfont}
\DeclareRobustCommand{\submissionthanksend}{\par\endgroup}

\title{\fontsize{16}{20}\selectfont\bfseries
Distributed Harmonic-Covariance Feedback for Hole Closure\\
in Planar Robot Swarms}

\author{Jitian Liu$^{1}$ and Subhrajit Bhattacharya$^{1}$%
\thanks{\submissionthanksbegin
$^{1}$Jitian Liu and Subhrajit Bhattacharya are with the Department of Mechanical
Engineering and Mechanics, Lehigh University, Bethlehem, PA, USA. Emails:
{\ttfamily\fontsize{9}{10}\selectfont \{jila24, sub216\}@lehigh.edu.}%
\submissionthanksend}}

\begin{document}
\begin{NoHyper}

\maketitle


\newcommand{\StateDampingDemoDetails}{For the 15-second demo, the saved frames had $\beta_1=2$ initially, decreasing to 1 by $t=8$~s. The first observed zero occurred at $t=9.9$~s, and every saved frame from that time through the terminal $t=15$~s frame had $\beta_1=0$.}
\newcommand{\StateDampingFigureProvenance}{The snapshot, \v{C}ech, and numerical area panels all use the same accepted state-dependent-damping trajectory.}

%

\begin{abstract}
We consider hole closure in a planar multi-agent system whose sensing disks form a time-varying \v{C}ech complex. We propose a harmonic-covariance feedback that uses harmonic edge signals to locate holes and converts those signals into agent forces. For distributed evaluation, we use a fixed-degree polynomial of the regularized Hodge Laplacian. This polynomial is evaluated through local operator applications. The closed loop has a unique global trajectory from every safe initial state. The model data and initial state are supplied with procedures for approximation to any requested accuracy with explicit error bounds. We assume that hole area is measurable and bounded, holes evolve continuously without splitting or merging, and only finitely many temporary new holes appear. Under these conditions, offline finite calculations produce rational lower bounds for the cumulative hole area on the time horizon without querying trajectory data. These bounds converge to that cumulative area, whether finite or infinite. If it is finite, holes close in the time-averaged sense: the average normalized hole area over every fixed-length time window tends to zero as time tends to infinity. In a 15-second numerical demo, all holes close by 9.9s and remain closed through the end.
\end{abstract}


%

\section{Introduction}
\label{sec:intro}

In nature, among groups of animals such as bird flocks and fish schools, local interaction
rules can generate collective motions
\cite{reynolds1987flocks,vicsek1995novel}. The same principle motivates
groups of mobile robots known as \emph{swarms}. For multi-robot systems,
sensor coverage is a basic problem: each robot senses
a disk around its position, and together the robots should cover their
operating region. As the robots move, the covered set can develop holes: regions enclosed by the
sensing disks that receive no coverage at all. These holes are exactly the holes of the \v{C}ech complex built from the
current robot positions \cite{desilva2007coverage,ghrist2008barcodes}.
The holes of a
moving cover are detected by persistent homology
\cite{desilva2007coverage,ghrist2008barcodes,edelsbrunner2010computational}
and verified without localization or with switched higher-order
Laplacians \cite{tahbazsalehi2010coverage,muhammad2007dynamic}.

Related work on repair and coverage includes simplicial coverage repair for robot teams \cite{derenick2010simplicial}, gradient descent on weighted
combinatorial Laplacians \cite{yadokoro2023weighted}, homology-based
repair of agricultural coverage \cite{rysz2025simplicial}, conformal
mapping of domains with holes onto multi-hole disks \cite{feng2026poriferous},
and diffusive coverage of bounded domains
\cite{elamvazhuthi2016coverage}. Related tools include cellular sheaves, which generalize weighted
Laplacians for distributed coordination
\cite{hansen2021opinion,hanks2025cellular}; simplicial filters for
processing on complexes \cite{yang2022simplicial,schaub2020random}; and
occupation-measure relaxations, which express dynamical constraints
and integral costs through measures and lead to linear optimization
problems over measures
\cite{lasserre2008nonlinear,henrion2024occupation}. 
The coverage-repair methods above do not provide a criterion for hole
closure based solely on the model and initial state. We derive such a
criterion for our distributed \emph{harmonic-covariance} controller.
The contributions of this paper are as follows:
\begin{itemize}
\item A feedback controller that converts the swarm \v{C}ech complex's harmonic edge flows, which localize the holes, into locally Lipschitz force laws on the agents (such that the controlled trajectory is unique), together with a distributed realization of the controller (Sections~\ref{sec:controller} and~\ref{sec:analysis}).
\item Under this paper's controlled dynamics and assumptions, we construct an offline sequence of rational lower bounds $\{c_T\}_T$. Each bound is produced by explicit finite calculations without querying trajectory data. The sequence converges to the cumulative hole area over the infinite time horizon. If its limit is finite, the hole area averaged over any fixed time window tends to zero as the window moves forward (Section~\ref{sec:analysis}).
\item A numerical demo in which two holes close within 10 seconds under the regularized harmonic feedback controller (Section~\ref{sec:benchmark}).
\end{itemize}

\begin{figure}[t]
\centering
\newlength{\figOnePanelHeight}
\setlength{\figOnePanelHeight}{0.35\columnwidth}
\begin{minipage}[b]{0.45\columnwidth}
\centering
\includegraphics[height=\figOnePanelHeight,keepaspectratio,
  trim=24bp 26bp 24bp 23bp,clip]{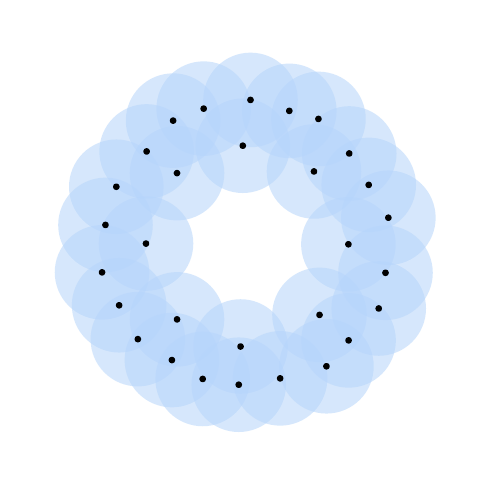}
\end{minipage}\hspace{0.2cm}
\begin{minipage}[b]{0.40\columnwidth}
\centering
\includegraphics[height=\figOnePanelHeight,keepaspectratio,
  trim=14bp 17bp 15bp 12bp,clip]{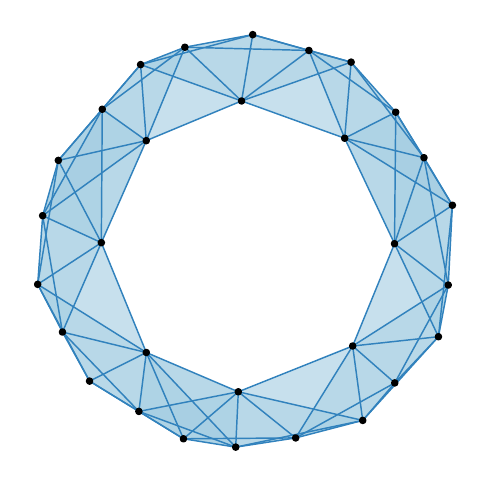}
\end{minipage}
\caption{Left: the union of sensing disks of the 30-agent annulus;
right: its strict \v{C}ech complex.}
\label{fig:schematic}
\end{figure}

The paper is organized as follows. Section~\ref{sec:problem} sets up the task and model. Section~\ref{sec:controller} designs the controller and its distributed realization. Section~\ref{sec:analysis} proves uniqueness of the state space trajectory under the regularized controller, and constructs $\{c_T\}_T$ from the swarm's initial state. Section~\ref{sec:benchmark} reports the numerical simulation. Appendix~\ref{app:constructions} records details to approximate the normalized hole areas in the swarm and the vector field yielded by the control law, whereas Appendix~\ref{app:proofs} supplies proofs to main theoretical results in this paper.

%

\section{Problem Formulation and Sensing Model}
\label{sec:problem}

We consider $N$ mobile sensing agents. Agent $i$ has position
$q_i(t)\in\mathbb R^2$ and velocity $v_i(t)\in\mathbb R^2$; we stack these
variables as $q,v\in\mathbb R^{2N}$ and write the state as $z=(q,v)$. Each
agent possesses a sensing disk centered at itself, and the control objective
is to remove the bounded regions left uncovered by the union of those disks.
Section~\ref{sec:holes}
defines these regions, which we call holes, and their cumulative cost.

\subsection{Sensing, the \v{C}ech complex, and the Hodge Laplacian}
\label{sec:cech}

Each agent senses the open disk of radius $r_s>0$ centered at its position.
The \emph{\v{C}ech 2-complex} records how these disks overlap. For a
configuration $q$, $\mathcal C(q)=\mathcal C(q;r_s)$ has vertices
$[N]=\{1,\ldots,N\}$, an edge $\{i,j\}$ whenever $\|q_i-q_j\|<2r_s$, and a triangle $\{i,j,k\}$ whenever the intersection of the three disks has strictly
positive measure, equivalently whenever the minimal enclosing circle of the
three points has radius strictly smaller than $r_s$ \cite{edelsbrunner2010computational}
(see Fig.~\ref{fig:schematic}).
At a fixed configuration, call the edges and triangles in $\mathcal C(q)$
\emph{active}. Orient each edge by $i<j$ and each triangle consistently with
its boundary edges. These choices define incidence (boundary) matrices
$B_1\in\mathbb R^{N\times E}$ and $B_2\in\mathbb R^{E\times F}$, where $E$
and $F$ count the active edges and triangles. Because $B_1$ and $B_2$
represent consecutive boundary maps, they satisfy $B_1B_2=0$. The
one-dimensional Hodge Laplacian is
$L_1=B_1^{\top}B_1+B_2B_2^{\top}\in\mathbb R^{E\times E}$, and
$\dim\ker L_1=\beta_1(q)$, the first Betti number of $\mathcal C(q)$,
equals the number of holes of the cover
\cite{schaub2020random}. As the disks move, their overlap pattern changes at
edge and triangle thresholds. To make the controller vary smoothly across
these changes, fix $0<\varepsilon_E<2r_s$ and define the scalar cutoff
\begin{equation}
\eta(x)=\begin{cases}\exp\left(-1/x\right),&x>0,\\0,&x\le0.\end{cases}
\label{eq:smooth-cutoff}
\end{equation}
For $e=\{i,j\}$, the edge weight is
\begin{equation}
a_e=a_{ij}=\frac{\eta(s_{ij})}{\eta(s_{ij})+\eta(1-s_{ij})},
\;\;
s_{ij}=\frac{2r_s-\|q_i-q_j\|}{\varepsilon_E}.
\label{eq:edge-weight}
\end{equation}
The weight $a_{ij}$ in \eqref{eq:edge-weight} lies in $[0,1]$ and equals one for
$\|q_i-q_j\|\le2r_s-\varepsilon_E$ and zero for
$\|q_i-q_j\|\ge2r_s$. Between these transition endpoints, it decreases
smoothly as the agents separate. It is locally Lipschitz, with
all distance derivatives vanishing at both thresholds. Thus crossing the
edge threshold creates no jump in the weighted operators.

To keep matrix dimensions fixed as simplices enter and leave the active
complex, use an \emph{ambient roster} containing every possible edge and
triangle. Its sizes are $\overline{E}=\binom{N}{2}$ and $\overline{F}=\binom{N}{3}$, with
oriented incidence matrices
$\overline{B}_1\in\mathbb R^{N\times\overline{E}}$ and
$\overline{B}_2\in\mathbb R^{\overline{E}\times\overline{F}}$. For each ambient triangle $f$,
let $R_f(q)\coloneqq\min_{c\in\mathbb R^2}\max_{\ell\in f}\|q_\ell-c\|$ be
the radius of the minimal enclosing circle of its three vertices. Fix
$0<\varepsilon_F<r_s$ and, using $\eta$ from \eqref{eq:smooth-cutoff}, define
\[
b_f(q)=\frac{\eta(s_f)}{\eta(s_f)+\eta(1-s_f)},
\;\; s_f=\frac{r_s-R_f(q)}{\varepsilon_F}.
\]
The weight $b_f$ is globally Lipschitz in $q$.
Both $a_e$ in \eqref{eq:edge-weight} and $b_f$ can be evaluated to any
prescribed accuracy with known error bounds.
The active restrictions of $\overline{B}_1,\overline{B}_2$ are $B_1,B_2$.

To preserve the corresponding identity
$\overline{B}_1\overline{B}_2=0$ under smooth weights, define
$D_1=\overline{B}_1\operatorname{diag}(a_e)\in\mathbb R^{N\times\overline{E}}$ and define
$D_2\in\mathbb R^{\overline{E}\times\overline{F}}$ entrywise by
\[
D_2[e,f]=\overline{B}_2[e,f]\,b_f\,a_e
\prod_{e'\subset f,\;e'\ne e}a_{e'}^2.
\]
For a triangle $f$ and a node $i$, entry
$(i,f)$ is
\begin{align*}
(D_1D_2)[i,f]
&=\sum_{e\subset f} D_1[i,e]\,D_2[e,f]\\
&=b_f\sum_{e\subset f}\overline{B}_1[i,e]\overline{B}_2[e,f]a_e^2
\prod_{e'\subset f,\;e'\ne e}a_{e'}^2\\
&=b_f\left(\prod_{e'\subset f}a_{e'}^2\right)
  (\overline{B}_1\overline{B}_2)[i,f]=0.
\end{align*}
The regularized Hodge Laplacian is
\[
L_w=D_1^{\top}D_1+D_2D_2^{\top}
\in\mathbb R^{\overline{E}\times\overline{E}}.
\]
A harmonic edge form is a vector in the kernel of the active-edge restriction
of $L_w$. When the weights are binary and
match the current \v{C}ech complex, this restriction of $L_w$ is $L_1$, so
its kernel has dimension $\beta_1(q)$.
Figure~\ref{fig:two-hole-harmonic-modes} shows a numerical orthonormal pair
from the weighted active-edge kernel of a connected two-hole complex.

\begin{figure}[t]
\centering
\includegraphics[width=0.49\linewidth]{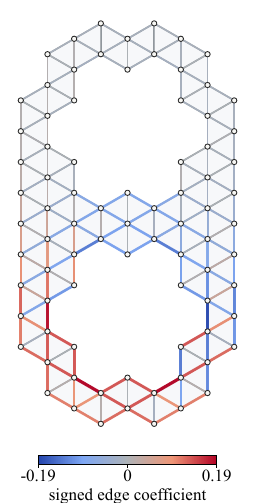}\hfill
\includegraphics[width=0.49\linewidth]{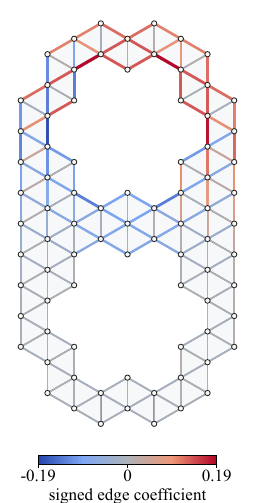}
\caption{Two orthonormal harmonic forms of the weighted active-edge
restriction of $L_w$ for the same connected two-hole complex. Triangles
are shown lightly; colored edges give signed coefficients on the common
scale, while gray denotes the zero level. Orientations are fixed
by $i<j$.}
\label{fig:two-hole-harmonic-modes}
\end{figure}

\subsection{Agents and dynamics}
\label{sec:dynamics}

The agents have second-order dynamics with three acceleration terms: repulsion
prevents hard collisions, damping reduces relative velocity along sensed
edges \cite{jadbabaie2003coordination,olfatisaber2006flocking,cucker2007emergent},
and the controller supplies harmonic-covariance feedback. For a damping gain
$\alpha\ge0$, define
\[
L_0^{\mathrm{reg}}(q)=\overline{B}_1\operatorname{diag}\left(a_e(q)^2\right)
\overline{B}_1^{\top},\;\; D(q)=\alpha L_0^{\mathrm{reg}}(q)\otimes I_2.
\]
The matrix $D(q)$ is symmetric positive semidefinite. An absent edge has
$a_e(q)=0$ and contributes nothing to the damping. Fix a hard distance
$a>0$, a repulsion
activation distance $b>a$, a gain $g>0$, and a smoothing parameter
$\tau_U>0$, and define the pairwise repulsive potential
\begin{equation}
U(q)=\sum_{i<j}\frac{g}{2}\,\frac{\rho_{ij}^4}{\tau_U^2+\rho_{ij}^2},
\;\;
\rho_{ij}=\frac{(b-\|q_i-q_j\|)_+}{\|q_i-q_j\|-a}.
\label{eq:repulsion}
\end{equation}
The ratio $\rho_{ij}$ combines finite-range activation, vanishing at and beyond
$b$, with a barrier as the distance decreases to $a$. The factor
$\rho_{ij}^4/(\tau_U^2+\rho_{ij}^2)$ turns the force on softly at $b$;
the potential retains the hard-boundary divergence used in
Section~\ref{sec:wellposed}. The agents obey
\[
\dot q=v,\;\; \dot v=-\nabla U(q)-D(q)v+u(q),
\]
where $u(q)\in\mathbb R^{2N}$ is the control input designed in
Section~\ref{sec:controller}.

\section{Controller Design}
\label{sec:controller}

\subsection{The harmonic-covariance force}
\label{sec:law}

Harmonic information is stored as scalar values on edges, while each agent
needs a two-dimensional force at a node. We first convert edge values into
node vectors and use these vectors to define force laws on each agent. The ideal
law uses the exact harmonic projector of the binary \v{C}ech complex;
Section~\ref{sec:distributed} replaces it with a fixed-degree polynomial of $L_w$ from Section~\ref{sec:cech}. Here ``exact'' and
``ideal'' refer only to the projector. 

Let $u(q)=F^{H}(q)$ be the ideal harmonic-covariance force. At a fixed
configuration with $E$ active edges, let $P$ be the orthogonal projector onto
$\ker L_1$ defined in Section~\ref{sec:cech} and choose an orthonormal basis
$\{\omega^{(k)}\}$ for that space.
A harmonic edge vector $\omega$ assigns a scalar $\omega_e$ to every 
edge. Orient $e=(p,q)$ by $p<q$ and set $d_e=q_q-q_p$. At endpoint $i$, let
$\hat t_e$ be the unit vector pointing along the edge away from $i$ and read $\omega_e$ with that
local orientation (i.e, for $e = (p, q)$, $\omega_{(p,q)} = -\omega_{(q,p)}$).

Fix $\lambda_R>0$ and define the \emph{lift} vector $y_i(\omega)$ of $\omega$ for agent $i$ as the unique
minimizer of the following least-squares problem
\begin{equation}
y_i(\omega)=\arg\min_{w\in\mathbb R^2}\;
\lambda_R\|w\|^2+\sum_{e\ni i}
\left(a_e\,w^{\top}\hat t_e-\frac{\omega_e}{\|d_e\|}\right)^2,
\label{eq:lift}
\end{equation}
where the factor $a_e\in[0,1]$ is the edge weight in
\eqref{eq:edge-weight}, and $\lambda_R\|w\|^2$ makes the problem strictly convex. Hence~\eqref{eq:lift} has a unique solution  $y_i=G_i^{-1}T_i$ with
\begin{equation}
G_i=\lambda_RI_2+\sum_{e\ni i}a_e^2\,\hat t_e\hat t_e^{\top},\;\;
T_i=\sum_{e\ni i}a_e\,\hat t_e\,\frac{\omega_e}{\|d_e\|}.
\label{eq:gram}
\end{equation}
In global edge coordinates, let $R_i \in\mathbb{R}^{2\times E}$ with
\[
(R_i)_{:,e}=
\begin{cases}
a_eG_i^{-1}d_e/\left\|d_e\right\|^2,&i\in e,\\
0,&i\notin e
\end{cases}.
\]
It then follows that $y_i(\omega^{(k)})=R_i\omega^{(k)}$. Because a single harmonic vector may concentrate
around one hole and miss others, we aggregate the lifts of the full orthonormal basis in the node ``covariance''
\begin{equation}
C_i=\sum_{k}y_i(\omega^{(k)})\,y_i(\omega^{(k)})^{\top}
=R_iPR_i^{\top}\succeq0,
\label{eq:cov}
\end{equation}
where $\sum_k\omega^{(k)}\omega^{(k){\mathsf T}}=P$,
$P\in\mathbb R^{E\times E}$; thus
$C_i\in\mathbb R^{2\times2}$. Division by $\|d_e\|$ puts the edge values on a
common distance scale, while the outer-product sum removes the sign and basis choices. Reading the endpoint covariances along an edge gives a nonnegative scalar score, which sets the strength of an attractive pair force on that edge. For $e=\{i,j\}$, gain $\kappa>0$, and common force limit $c_H>0$, define
\begin{align}
s_e
&=(1/2)\hat t_e^{\top}\left(C_i+C_j\right)\hat t_e\ge0,
\label{eq:score}\\
F_i^{e}
&=\kappa a_es_e(q_j-q_i),
\;\; F_j^{e}=-F_i^{e},\;\;
F_i^{\mathrm{raw}}=\sum_{e\ni i}F_i^{e},\notag\\
\lambda
&=\frac{c_H}{\max\left\{c_H,\max_\ell\|F_\ell^{\mathrm{raw}}\|\right\}},
\;\; F_i^{H}=\lambda F_i^{\mathrm{raw}}.
\label{eq:cap}
\end{align}
Equation~\eqref{eq:score} gives the edge score, while \eqref{eq:cap} assembles
the incident pair forces and scales every node by the same factor $\lambda$.
Since $C_i,C_j\succeq0$, the score is nonnegative and independent of edge
orientation. The weight $a_e$ attenuates the force near loss of support (see
Fig.~\ref{fig:forces}). Each pair is equal and opposite and collinear with its
edge, so the harmonic force has zero net force and zero torque. The common
scaling by $\lambda$ in \eqref{eq:cap} preserves these properties and enforces
$\|F_i^H\|\le c_H$ at every node.

\begin{figure}[t]
\centering
\newlength{\figTwoPanelHeight}
\setlength{\figTwoPanelHeight}{0.42\columnwidth}
\begin{minipage}[t]{0.5\columnwidth}
\centering
\includegraphics[height=\figTwoPanelHeight,keepaspectratio,
  trim=14bp 17bp 1bp 12bp,clip]{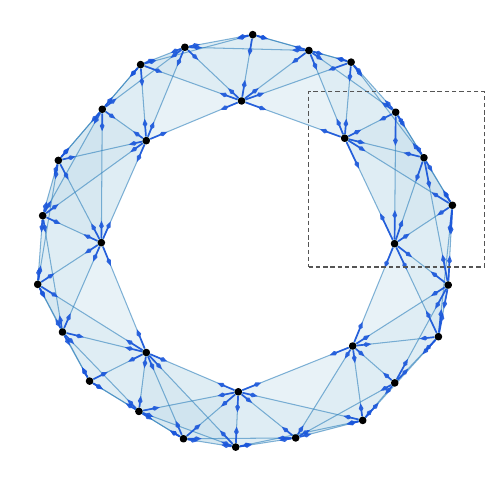}\\[1pt]
\footnotesize Edge forces $F_i^e$
\end{minipage}\hfill
\begin{minipage}[t]{0.42\columnwidth}
\centering
\includegraphics[height=\figTwoPanelHeight,keepaspectratio,
  trim=0bp 0bp 40bp 0bp,clip]{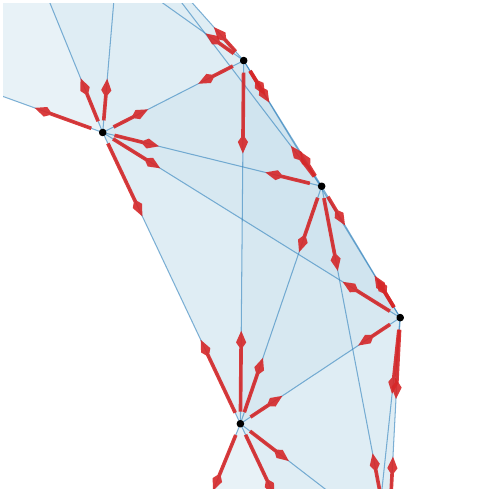}\\[1pt]
\footnotesize Zoomed-in detail
\end{minipage}
\caption{Edge forces of the harmonic-covariance law.}
\label{fig:forces}
\end{figure}

\begin{lemma}
\label{lem:properties}
For every $\overline{q}\in\mathbb R^2$,
$\sum_i(q_i-\overline{q})^{\top}F_i^{H}\le0$.
\end{lemma}
What Lemma~\ref{lem:properties} implies, in plain language, is that the net harmonic force tends to pull all agents together. Yet the lemma guarantees neither monotonic decrease of the hole areas nor their eventual closure.

\subsection{Distributed realization of harmonic-covariance forces}
\label{sec:distributed}

Direct computation of $P$ in \eqref{eq:cov} requires eigenvectors in the
kernel of the full matrix $L_1$, which requires centralized computation. For a
fixed binary $L_1$, $P=\lim_{\tau\to\infty}e^{-\tau L_1}$. The exponential leaves the zero eigenspace unchanged and attenuates every other subspace. This limit also applies to $L_w$ and one may use $e^{-\tau L_w}$ as an approximate projector onto $\ker L_w$. Evaluating $e^{-\tau L_w}$ directly, however, still requires centralized computation. We therefore replace it by a polynomial of degree $d$, which requires $d$ local applications of $L_w$~\cite{shuman2011chebyshev,yang2022simplicial,zhang2020barriers}. The factor
$\lambda$ in \eqref{eq:cap} is computed separately. We choose a Chebyshev
expansion because it treats the full spectral interval uniformly. A Taylor
expansion is centered at a single point and requires a separate remainder
estimate to control its error over the full interval.

The Chebyshev interval requires a state-independent spectral upper bound, i.e., some $B\ge\lambda_{\max}(L_w)$. Since the entries of $D_1$ and $D_2$ have magnitude at most one, for $N\ge2$ the row-sum
bounds give $\left\|D_1^{\top}D_1\right\|_{\infty}\le2N-2$ and
$\left\|D_2D_2^{\top}\right\|_{\infty}\le3\left(N-2\right)$.
Consequently,
$\lambda_{\max}(L_w)\le\|L_w\|_{\infty}
\le\|D_1^\top D_1\|_\infty+\|D_2D_2^\top\|_\infty\le 5N-8\eqqcolon B$.
Since $L_w$ is positive semidefinite, its spectrum lies in $[0,B]$, and the
affine map $2\lambda/B-1$ sends this interval to $[-1,1]$. Accordingly, define
$\widetilde L_w=2L_w/B-I$. Put $\sigma=\tau B/2$. Let $T_j$ be the Chebyshev polynomial
defined by $T_j(\cos\theta)=\cos(j\theta)$, and let $I_j$ be the modified
Bessel function of the first kind. Define
\[
c_j=
\begin{cases}
e^{-\sigma}I_0\left(\sigma\right),&j=0,\\
2e^{-\sigma}\left(-1\right)^jI_j\left(\sigma\right),&j\ge1,
\end{cases}
\]
In the Jacobi--Anger expansion of $e^{-\tau\lambda}$ \cite{nistDLMF10352}, its
degree-$d$ truncation is
\[
p_d(\lambda)
=\sum_{j=0}^{d}c_j\,T_j\left(\frac{2\lambda}{B}-1\right).
\]
The symmetric matrix $H_d=p_d(L_w)$ approximates the exponential. Its
parameters $\tau,B,d,c_j$ are fixed and can be evaluated offline, so no online
spectral estimate is needed.

We aim to span $\ker L_w$ and generate edge signals by approximately projecting random initial vectors onto $\ker L_w$. We initialize $Z\in\mathbb R^{\overline{E}\times K}$ once at random with $K>1$ columns on the fixed ambient roster of Section~\ref{sec:cech}, and set $Y=e^{-\tau L_w}Z\approx H_dZ\eqqcolon Y_d$. Computationally, $Y_d$ may be evaluated by the following recursive procedure:
\[
\begin{aligned}
X_0=Z,\;\; X_1&=\widetilde L_wZ\;\;(d\ge1),\\
X_{r+1}=2\widetilde L_wX_r&-X_{r-1},
\;\; 1\le r<d,\\
Y_d=\sum_{r=0}^{d}c_rX_r
&=p_d(L_w)Z=H_dZ.
\end{aligned}
\]
All these computations may be performed locally.

To convert the columns of $Y_d$ back to node quantities, extend the active-edge map $R_i$ to $\mathbb R^{2\times\overline{E}}$ by assigning zero columns to absent edges. The lift of $Y_d$'s $k$-th column is $y_{i,k}=R_i(Y_d)_{:,k}$, and their
outer products give the probe covariance
\begin{equation}
\widehat C_i=\sum_{k}y_{i,k}y_{i,k}^{\top}
=R_iH_dZZ^{\top}H_d^{\top}R_i^{\top} \succeq 0.
\label{eq:sketchcov}
\end{equation}
Here $y_{i,k}$ is analogous to $y_i(\omega^{(k)})$ in
Section~\ref{sec:law}. Since $\widehat{C}_i\succeq 0$, substituting it into~\eqref{eq:score} and~\eqref{eq:cap} preserves nonnegative edge scores. The rules in \eqref{eq:cap} are otherwise unchanged, except $(C_i+C_j)$ is replaced by $(\widehat{C}_i + \widehat{C}_j)$. The factor $\lambda$ uses the maximum raw force norm over all nodes, so the distributed realization assumes a connected communication graph over which this maximum can propagate. Damping is likewise local because $L_0^{\mathrm{reg}}(q)$ couples each agent only to its current one-hop
neighbors.


\section{Well-Posedness and Bounds on Cumulative Hole Areas}
\label{sec:analysis}

After the regularized controller is defined, it is natural to raise the following questions:
\begin{enumerate}[noitemsep]
    \item Given a safe initial state, does the controller yield \emph{one unique} trajectory?
    \item Given an initial state, without simulation, can we determine whether the holes will eventually close under the controlled dynamics?
\end{enumerate}
Section~\ref{sec:wellposed} gives the first question a positive answer, and Section~\ref{sec:initial-state-construction} answers the second question with an offline construction of a rational sequence $\{c_T\}_T$ converging to the all-time cumulative hole area, with each term obtained by finite calculation.

\subsection{Holes and the hole-area cost}
\label{sec:holes}

All sensing disks are open, and their union at configuration $q$ is
$A(q)=\bigcup_{i\in[N]}B(q_i,r_s)$. The \emph{holes} are the bounded connected
components of $\mathbb R^2\setminus A(q)$. Let $H(q)$ be their union and let
$|\cdot|$ denote two-dimensional Lebesgue measure. A terminal hole count does
not record how much uncovered area persisted before the terminal time. We
therefore use the normalized instantaneous area and its all-time cumulative
cost:
\begin{equation}
h(q)=\frac{|H(q)|}{(8Nr_s)^2};\;
C=\int_0^{\infty}h(q(t))\,\mathrm dt\in[0,\infty).
\label{eq:cost}
\end{equation}
If our proposed controller is successful, our objective will be to show that $C$ is finite.
Every hole lies in the convex hull of one
sensing-graph component. A component with $k_j$ vertices has center diameter
at most $2r_s(k_j-1)$ and fits in a box of area at most
$4r_s^2(k_j-1)^2$; the normalization by $(8Nr_s)^2$ therefore gives
$0\le h(q)\le H_N$, where
$H_N\coloneqq (N-1)^2/(16N^2)< 1/16$. Section~\ref{sec:cech}
records the correspondence between the holes and $\beta_1(q)$.

\subsection{Assumptions}

The admissible
motion class satisfies the following standing assumptions:
\begin{enumerate}[label=A\arabic*]
\item $\|q_i(t)-q_j(t)\|>a$ for all $i<j$ and all $t\ge0$.
\item The normalized hole area $h(q(t))$ is well defined, measurable, and bounded for
all $t\ge0$. 
\item Hole components grow and shrink continuously and do not split
or merge. 
\item At most finitely many holes appear after time zero, and every
such hole is \emph{spurious} (present only over a bounded time interval).
\end{enumerate}
Without splitting or merging, each hole keeps a well-defined identity
over its lifetime, and since every hole appearing after time zero is
spurious, their cumulative total area is finite, and therefore will not affect the finiteness of $C$. The measurability and
boundedness assumptions make the all-time integral \eqref{eq:cost} well
defined.

\subsection{Well-posedness}
\label{sec:wellposed}

For our controlled dynamics, the natural state space is the safe open set $\Omega=\left\{(q,v):\min_{i<j}\|q_i-q_j\|>a\right\}$. 

\begin{proposition}
\label{prop:wellposed}
For every safe initial state $z_0\in\Omega$, the controlled dynamics of Section~\ref{sec:dynamics}
with the regularized law of
Section~\ref{sec:distributed}
has a unique solution $z:[0,+\infty)\to\Omega$. That means no hard
collision occurs in finite time.
\end{proposition}

\begin{proof}
The proof has two parts. We first show that the vector field is locally
Lipschitz. Indeed, the edge weights and entries of $L_w$ have this property
on $\Omega$, and the fixed-degree polynomial filter preserves it. In
\eqref{eq:cap}, the maximum raw-force norm has the same regularity, while the
denominator is bounded below by $c_H>0$. On the safe open set, the
denominators in the repulsive potential are bounded away from zero locally,
so $U$ is locally $C^{1,1}$; the damping term $D(q)v$ is a finite sum using
the same weights. Hence the full vector field
$z\mapsto\left(v,F^{H}(q)-\nabla U(q)-D(q)v\right)$ is locally Lipschitz.
Picard-Lindel\"of therefore gives a unique solution for a sufficiently short
time after $t=0$.

We next show that this solution can be continued to any time. Fix any finite
horizon $T>0$. Define the energy by
$E(q,v)=\lVert v\rVert^2/2+U(q)$. Its damping dissipation satisfies
$v^{\top}D(q)v=\alpha\sum_{\{i,j\}}a_{ij}(q)^2
\lVert v_i-v_j\rVert^2\ge0$. Moreover, \eqref{eq:cap} gives, via $\|F_i^H\|\leq c_H$, that
\[
v^{\top}F^H
=\sum_i v_i^{\top}F_i^H
\le c_H\sum_i\lVert v_i\rVert
\le \frac12\lVert v\rVert^2+\frac{Nc_H^2}{2},
\]
where the last inequality comes from $(c_H - \|v_i\|)^2 \geq 0$. Thus, with
$C_H=Nc_H^2/2$,
\[
\frac{\mathrm dE}{\mathrm dt}
=v^{\top}F^H-v^{\top}D(q)v
\le \frac12\lVert v\rVert^2+C_H
\le E+C_H.
\]
Consequently, for $t\le T$,
$E(t)\le e^t\left(E(0)+C_H\right)-C_H
\allowbreak\le e^T\left(E(0)+C_H\right)-C_H$. Hence $E$, and therefore
$\lVert v\rVert$, is bounded up to time $T$, and $q$ will not escape
to infinity before $T$. The finiteness of $E$ also enforces that $U(q)$ does not blow up, i.e., one always has $\|q_i - q_j\| > a$. Hence the state stays in a compact subset of $\Omega$ before any finite endpoint. The vector field is bounded on that compact set, so the state has a limit in $\Omega$ at the endpoint. Starting from this limiting state, the same local existence result continues the solution. This gives a unique solution on every finite interval $[0,T]$. Since $T$ is arbitrary, letting $T\to\infty$ gives existence for all $t\ge0$.
\end{proof}

\subsection{Construction of cumulative hole area lower bounds}
\label{sec:initial-state-construction}
We now fix the feedback law and examine whether its cumulative hole area can be
determined from the model and initial state. Fix the safe initial state
$z_0=(q_0,v_0)$ and write the controlled dynamics as
\begin{equation}
\begin{aligned}
\dot z&=F(z),\;\;F(q,v)=\left(v,f(q,v)\right),\\
f(q,v)&=-\nabla U(q)-D(q)v+F^H(q).
\end{aligned}
\label{eq:closed-loop-state}
\end{equation}
For a positive integer horizon $T$, the physical quantity of interest is the
actual hole-area cost from Section~\ref{sec:holes},
\[
J_T=\int_0^T h\left(q(t)\right)\,\mathrm dt.
\]

\subsubsection{Outline of bound construction}

Directly using $h$ is difficult because it need not vary continuously with the agent positions. A small displacement can connect a hole to the exterior, so that region is no longer counted by $h$, even though it remains uncovered. We therefore approach $h$
from below by continuous functions. For each positive integer $T$, it is
possible to construct
$q_T:\Omega\to[0,H_N]$ independent of $v$
such that
\begin{equation}
0\le q_T(z)\le q_{T+1}(z)\le h(q),
\;\; q_T(z)\uparrow h(q).
\label{eq:area-approximation-properties}
\end{equation}
Here $\uparrow$ denotes non-decreasing convergence.
Appendix~\hyperref[par:area-lower-approximation]{\ref*{app:constructions}}
constructs this graph-based family $q_T$. Define
\begin{equation}
A_T=\int_0^T q_T\left(z(t)\right)\,\mathrm dt 
\le J_T\le C
\label{eq:approximated-area-cost}
\end{equation}
such that $A_T\uparrow C$. For fixed
$T$, it is possible to compute rational bounds
\begin{equation}
L_{T,m}\le A_T\le U_{T,m}.
\label{eq:main-text-bounds}
\end{equation}
The index $m$ will tighten a fixed-$T$ finite calculation, whereas increasing $T$
both extends the time horizon and improves the lower approximation $q_T$ of
$h$. Consequently, the stopping tolerance controls the upper and lower bounds of $A_T$. The details proceed in three steps. Section~\ref{sec:recovering-trajectory}
constructs integral constraints that recover the cumulative behavior of the
unique trajectory. Section~\ref{sec:approximating-field} replaces the controlled
field in those constraints by finite rational data. Section~\ref{sec:finite-area-optimization}
then computes $L_{T,m}$ and $U_{T,m}$ in~\eqref{eq:main-text-bounds}.



\subsubsection{Recovering the trajectory through integral constraints}
\label{sec:bound-construction}
\label{sec:recovering-trajectory}

The construction begins by enclosing every possible state that the true trajectory can
visit before time $T$. The energy, displacement, and repulsion estimates of
Section~\ref{sec:wellposed} give rational bounds $Q_T,V_T$ and
$r_T>a$, depending only on $F,z_0,T$, such that the trajectory lies in a compact subset
\begin{equation}
K_T
=\left\{\,(q,v)\in\Omega:
\begin{aligned}
&\lVert q\rVert\le Q_T,\;\; \lVert v\rVert\le V_T,\\
&\min_{i<j}\lVert q_i-q_j\rVert\ge r_T
\end{aligned}
\,\right\}\subset\Omega.
\label{eq:finite-horizon-domain}
\end{equation}
Write $X_T=[0,T]\times K_T$.

We now add constraints to rule out candidate trajectories that are
inconsistent with the controlled dynamics. 
We use scalar test functions to impose the controlled dynamics as integral
constraints and narrow the feasible set.
By Lemma~\ref{lem:occsingleton} and the uniqueness established in
Proposition~\ref{prop:wellposed}, the full family of constraints recovers the
cumulative integrals of the true trajectory.

Choose a scalar rational function $\psi(t,q,v)$ with rational coefficients, $C^1$ on a neighborhood of $X_T$, and satisfying $\psi(T,\cdot)=0$, independently of the unknown trajectory. By the chain rule,
$\mathrm d\psi\left(t,z(t)\right)/\mathrm dt
=\partial_t\psi+v^{\top}\nabla_q\psi+f^{\top}\nabla_v\psi$, with the right-hand side evaluated at $\left(t,z(t)\right)$. The fundamental theorem of calculus then yields the integral constraints where the right hand side only involves the initial state:
\begin{equation}
\begin{aligned}
&\int_0^T\left(\partial_t\psi+F^{\top}\nabla_z\psi\right)
 \left(t,z(t)\right)\,\mathrm dt\\
&\;\;=\underbrace{\psi\left(T,z(T)\right)}_{=0}-\psi(0,z_0)
=-\psi(0,z_0).
\end{aligned}
\label{eq:occupation-liouville}
\end{equation}
To obtain a sequence of finite constraint sets, we enumerate all
polynomials with rational coefficients as $p_1,p_2,\ldots$ and define
$\psi_\ell(t,z)=p_\ell(t,z)-p_\ell(T,z)$, so
$\psi_\ell(T,z(T))=0$.
For each index $\ell$, define
\begin{equation}
\begin{aligned}
a_\ell=\partial_t\psi_\ell+F^{\top}\nabla_z\psi_\ell\\
\implies\int_0^T a_\ell\left(t,z(t)\right)\,\mathrm dt=-&\psi_\ell(0,z_0)\eqqcolon -\beta_\ell.
\end{aligned}
\label{eq:test-polynomial-identities}
\end{equation}

\subsubsection{Approximating the controlled vector field}
\label{sec:approximating-field}

The integral identities in \eqref{eq:test-polynomial-identities} constrain
cumulative behavior, but they are not yet finite rational data: $a_\ell$
contains the generally non-polynomial field $F$, and $\beta_\ell$ may not be
rational. We make both quantities finite rational data by approximating the
field and the initial-value term. The integer $k$ below records the requested
field accuracy.

\begin{proposition}
\label{prop:field-approximation}
Suppose uniform rational approximations of the model data, with explicit error bounds, can be constructed on compact sets to any prescribed accuracy. For every integer $k\ge1$, a finite construction from the model data alone produces a continuous piecewise-affine field $F_k$ with rational coefficients on a neighborhood of $K_T$, satisfying $\lVert F_k-F\rVert_{\infty,K_T}\le2^{-k}$.
\end{proposition}
Appendix~\ref{app:field-construction} proves
Proposition~\ref{prop:field-approximation} by explicit construction.
For this $F_k$, choose a rational
$\Gamma_\ell\ge\lVert\nabla_z\psi_\ell\rVert_{\infty,X_T}$ and a rational
$\beta_{\ell k}$ satisfying
$\left|\beta_{\ell k}-\beta_\ell\right|\le2^{-k}$. Set
\begin{equation}
a_{\ell k}=\partial_t\psi_\ell+F_k^{\top}\nabla_z\psi_\ell,
\;\;
\left\lVert a_{\ell k}-a_\ell\right\rVert_{\infty,X_T}
\le \Gamma_\ell2^{-k}.
\label{eq:test-integrand-approximation}
\end{equation}
At stage $m$, compute $a_{\ell k}$ and $\beta_{\ell k}$ for every pair $1\le\ell,k\le m$. We use these quantities to construct constraints for the finite optimization problem below to compute the lower and upper bounds of $A_T$.

\subsubsection{Finite optimization of the area bounds}
\label{sec:finite-area-optimization}
To formulate this optimization problem, we approximate $q_T$ by a continuous rational piecewise polynomial $\widehat q_{T,m}$ with
\begin{equation}
\left\lVert\widehat q_{T,m}-q_T\right\rVert_{\infty,K_T}\le2^{-m}.
\label{eq:objective-approximation}
\end{equation}
Appendix~\ref{app:objective-construction} gives this construction.
Stage $m$ now has $m^2$ constraints and one approximate-area objective.
For the true trajectory, the exact identity
$\int_0^T a_\ell(t,z(t))\,\mathrm dt=-\beta_\ell$ gives, for each $1\le \ell, k\le m$,
\begin{equation}
\begin{aligned}
\int_0^T &a_{\ell k}(t,z)\mathrm dt+\beta_{\ell k}
=\int_0^T\left(a_{\ell k}-a_\ell\right)(t,z)\mathrm dt+\beta_{\ell k}-\beta_\ell\\
&\implies\left|\int_0^T a_{\ell k}(t,z)\mathrm dt+\beta_{\ell k}\right|
\le\left(T\Gamma_\ell+1\right)2^{-k}.
\end{aligned}
\label{eq:main-test-residual}
\end{equation}
Lemma~\ref{lem:staticeval} in Appendix~\ref{app:proofs} shows that there exists a collection of points $x_r=(t_r,z_r)\in X_T$ and weights $w_r\ge0$, with total weight $T$, that can simultaneously represent the $m^2$ constraints and the one objective value, using
at most $M=m^2+2$ points. For every $1\le\ell,k\le m$, this gives
\begin{equation}
\begin{aligned}
\int_0^T a_{\ell k}(t,z(t))\,\mathrm dt
&=\sum_{r=1}^{M}w_ra_{\ell k}(x_r),\\
\int_0^T\widehat q_{T,m}(z(t))\,\mathrm dt
&=\sum_{r=1}^{M}w_r\widehat q_{T,m}(z_r),
\end{aligned}
\label{eq:common-finite-representation}
\end{equation}
where $r$ indexes the common representing
points. The desired optimization problems then follow from the constraints and approximated cost $\widehat q_{T,m}$, for each $1\le\ell,k\le m$:
\begin{equation}
\begin{aligned}
&S_{T,m}(x,w)=\sum_{r=1}^{M}w_r\widehat q_{T,m}(z_r),\\
&j^-_{T,m}=\min_{x,w}S_{T,m}(x,w),\;\;
j^+_{T,m}=\max_{x,w}S_{T,m}(x,w),\\[-1mm]
&\text{s.t.}\;x_r=(t_r,z_r)\in X_T,\;\; w_r\ge0,\;\;
\sum_{r=1}^{M}w_r=T,\\
&\left|\sum_{r=1}^{M}w_r a_{\ell k}(x_r)+\beta_{\ell k}\right|
\le\left(T\Gamma_\ell+1\right)2^{-k}.
\end{aligned}
\label{eq:finite-static-constraints}
\end{equation}
The last line of~\eqref{eq:finite-static-constraints} directly comes from~\eqref{eq:main-test-residual} and~\eqref{eq:common-finite-representation}.
Minimizing and
maximizing $S_{T,m}$ subject to the constraints in~\eqref{eq:finite-static-constraints}
therefore gives the bounds
\begin{equation}
j^-_{T,m}-T2^{-m}\le A_T\le j^+_{T,m}+T2^{-m}.
\label{eq:static-extrema-bracket}
\end{equation}
The term $T2^{-m}$ accounts for replacing $q_T$ by
$\widehat q_{T,m}$. Note that~\eqref{eq:finite-static-constraints} should not be solved via ordinary numerical optimization, which may return a feasible or locally optimal candidate without establishing the global bounds needed in \eqref{eq:static-extrema-bracket}. A feasible objective value bounds the minimum from above and the maximum from below; \eqref{eq:static-extrema-bracket} requires the opposite.
Instead, for fixed $T,m$, 
we perform a binary search on 
a threshold value, $\theta$. 
Real quantifier elimination \cite[Ch.~14]{basu2006algorithms} decides feasibility with $S_{T,m}<\theta$ or $S_{T,m}>\theta$, and these answers update rational enclosures of the minimum and maximum. Repeating at rational midpoints encloses both extrema to any demanded accuracy. Round $j^-_{T,m}-T2^{-m}$ downward and $j^+_{T,m}+T2^{-m}$ upward to rational numbers $L_{T,m}$ and $U_{T,m}$, respectively, with each rounding error at most $2^{-m}$. This gives \eqref{eq:main-text-bounds}. Applying Lemma~\ref{lem:staticeval} with $g(t,z)=q_T(z)$ and $g_m(t,z)=\widehat q_{T,m}(z)$ shows that these bounds satisfy $U_{T,m}-L_{T,m}\to0$ for each fixed $T$. With $c_0=0$, define
\begin{equation}
\begin{aligned}
m_T&=\min\left\{m\ge1:U_{T,m}-L_{T,m}<2^{-T}\right\},\\
(\ell_T,u_T)&=\left(L_{T,m_T},U_{T,m_T}\right),\;\;
c_T=\max\left\{c_{T-1},0,\ell_T\right\}.
\end{aligned}
\label{eq:stopping-and-outer-update}
\end{equation}
For fixed $T,m$, the binary search above refines the numerical bounds on the two extrema. If $L_{T,m},U_{T,m}$ do not satisfy \eqref{eq:stopping-and-outer-update}, increase $m$ and solve \eqref{eq:finite-static-constraints} again. The gap convergence above proves that this loop terminates and yields $\ell_T\le A_T\le u_T$ with $u_T-\ell_T<2^{-T}$.

By construction, $\mathbf 1_{[0,T]}(t)q_T(z(t))$ increases pointwise to $h(q(t))$. Monotone convergence theorem then gives $A_T\uparrow C$, also when $C=+\infty$. If $S\le T$, then $\ell_S\le A_S\le A_T$. Hence the running maximum in \eqref{eq:stopping-and-outer-update} is non-decreasing and satisfies $A_T-2^{-T}<\ell_T\le c_T\le A_T\le C$. Therefore $c_T\uparrow C$. By this point, one completes the proof to the following result:

\begin{theorem}
\label{thm:lower}
Assume that we can calculate rational approximations of the model data and the safe initial state $z_0$ to any prescribed accuracy, with explicit error bounds. For every\hspace*{0.25pt} integer $T\ge1$, the construction \eqref{eq:stopping-and-outer-update} produces a rational number $c_T$ in finitely many steps from these data alone, and
\begin{equation}
c_T\uparrow C=\int_0^\infty h(q(t))\,\mathrm dt,
\label{eq:cn}
\end{equation}
with $C=+\infty$ allowed.
\end{theorem}

By Theorem~\ref{thm:lower} and \eqref{eq:cn}, $C<\infty$ exactly when
$\{c_T\}$ is bounded. If $C<\infty$, then for every fixed $\Delta>0$,
\[
0\le\frac{1}{\Delta}\int_t^{t+\Delta}h(q(s))\,\mathrm ds
\le\frac{1}{\Delta}\int_t^\infty h(q(s))\,\mathrm ds\to 0
.
\]
Pointwise convergence of $h(q(t))$ to zero, however, is not guaranteed unless $h(q(t))$ is uniformly continuous, by Barbalat's lemma.


\section{Numerical Demo}
\label{sec:benchmark}

This section shows how the controller behaves on a 60-agent configuration
with two initial holes. The snapshots show the agent motion and the numerical
\v{C}ech complex, while the area plot follows the two holes that are present
at the initial time. 

\begin{figure*}[t]
\centering
\begin{minipage}[b]{0.235\textwidth}
\centering
\footnotesize $t=0.0$~s\\[1pt]
\includegraphics[width=\linewidth]{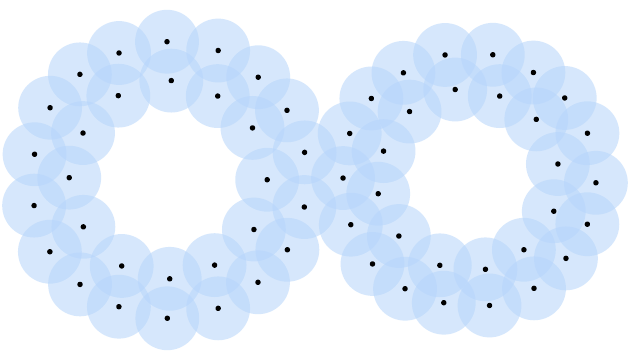}\\[1pt]
\includegraphics[width=\linewidth]{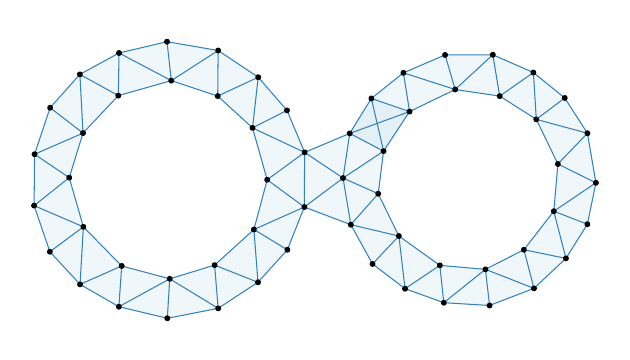}
\end{minipage}\hspace{0.1cm}
\begin{minipage}[b]{0.235\textwidth}
\centering
\footnotesize $t=5.0$~s\\[1pt]
\includegraphics[width=\linewidth]{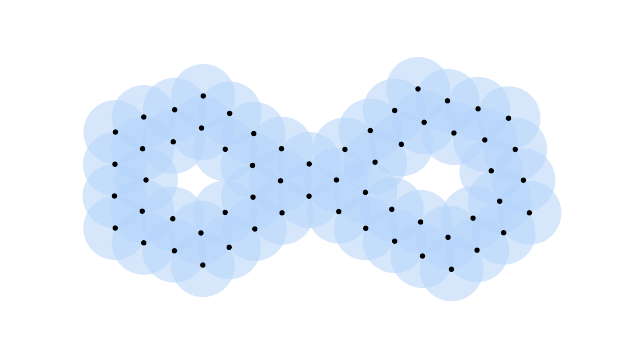}\\[1pt]
\includegraphics[width=\linewidth]{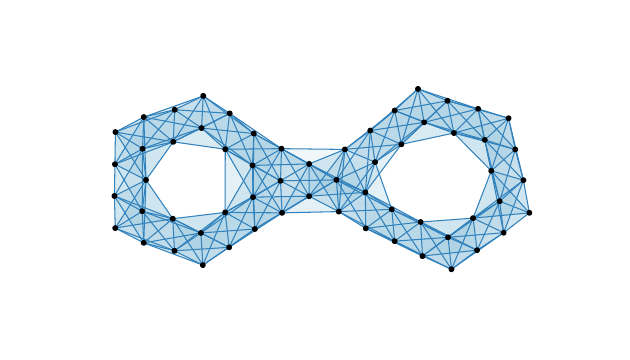}
\end{minipage}\hspace{0.1cm}
\begin{minipage}[b]{0.235\textwidth}
\centering
\footnotesize $t=10.0$~s\\[1pt]
\includegraphics[width=\linewidth]{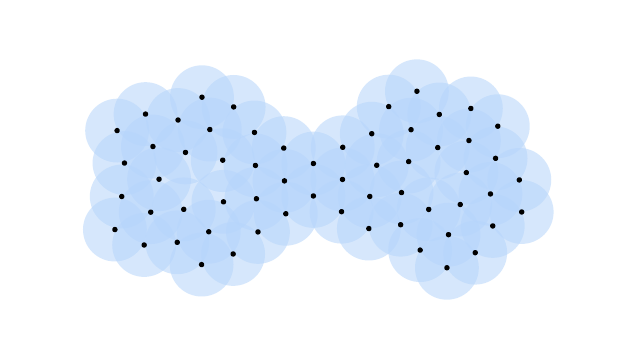}\\[1pt]
\includegraphics[width=\linewidth]{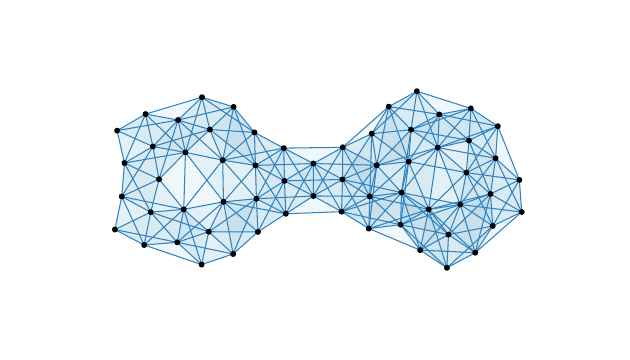}
\end{minipage}\hspace{0.1cm}
\begin{minipage}[b]{0.235\textwidth}
\centering
\footnotesize $t=15.0$~s\\[1pt]
\includegraphics[width=\linewidth]{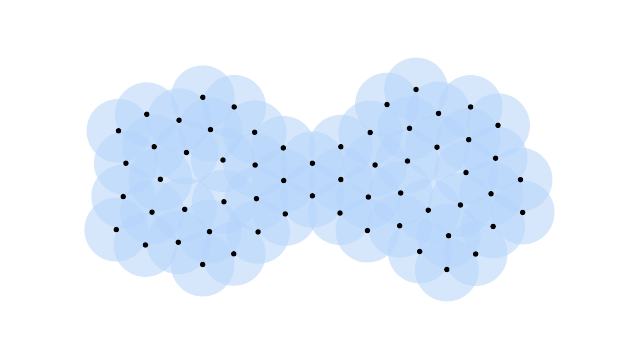}\\[1pt]
\includegraphics[width=\linewidth]{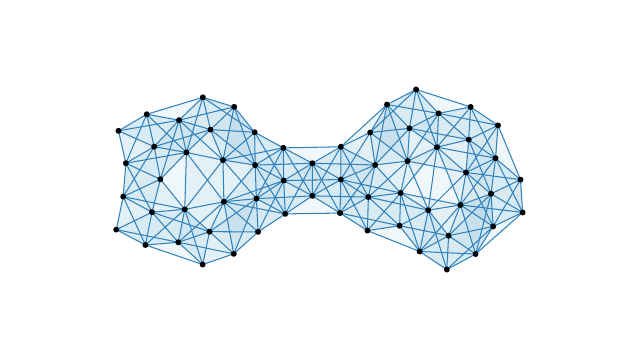}
\end{minipage}
\caption{The two-hole peanut with state-dependent damping. Each column shows the configuration and numerical
\v{C}ech complex from the same accepted trajectory; the terminal saved frame
has $\beta_1=0$.}
\label{fig:snapshots}
\end{figure*}

\begin{figure}[t]
\centering
\includegraphics[width=\columnwidth]{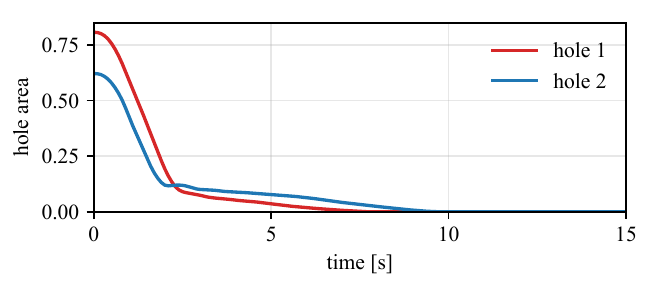}
\caption{Areas (coordinate units squared) of the two
tracked initial holes over the 15~s run, computed using open disks. Both
curves read zero by the end of the run.}
\label{fig:holeareas}
\end{figure}

The initial positions form the two-hole peanut shown on the left of Fig.~\ref{fig:snapshots}. Table~\ref{tab:benchmark-parameters} lists the parameters and integrator settings, where the Chebyshev polynomial degree $d=343$ is determined by the stated polynomial-tail tolerance. The run integrates the controller in Section~\ref{sec:distributed}. The implementation evaluates the same finite-degree local operators with centralized single-process sparse operations. It emulates the algebra of the distributed recurrence but does not measure message delays or execute a networked message-passing protocol.

\begin{table}[t]
\centering
\caption{Parameters for the numerical demo}
\label{tab:benchmark-parameters}
\begin{tabular}{@{}p{0.52\columnwidth}p{0.40\columnwidth}@{}}
\toprule
Setting & Value \\
\midrule
Agents; initial number of holes & $N=60$; 2 \\
Sensing-disk radius & $r_s=0.234$ \\
Edge/triangle regularization width & $2.5\times 10^{-4}$ \\
Local ridge $\lambda_R$; gain $\kappa$; force cap $c_H$
  & $0.05$; $8$; $0.8$ \\
Number of columns in $Z$; random seed & $K=48$; $419$ \\
$\tau$; $d$ & $60$; $343$ \\
Polynomial error tolerance & $5\times 10^{-3}$ \\
Repulsion distances & $a=0.18$; $b=0.24$ \\
Repulsion gain; smoothing & $0.5$; $\tau_U=0.25$ \\
Damping gain & $\alpha=1.2$ \\
Integrator & DOP853 \\
Absolute; relative tolerance & $10^{-9}$; $10^{-6}$ \\
Maximum step & $0.05$ \\
Horizon; saved-frame grid & $15$s; $0.1$s \\
\bottomrule
\end{tabular}
\end{table}

\subsection{Results}

\StateDampingDemoDetails\ 
Fig.~\ref{fig:snapshots} should be read from left to right: the agents move
inward around the two uncovered regions, and
the displayed complex has no holes by the final
saved frame. Fig.~\ref{fig:holeareas} reports the raw areas of the two tracked initial holes. The two tracked curves first read zero respectively at $7.8$ and $9.8$s and remain zero afterwards.

\section{Conclusion and Future Work}

We developed a harmonic-covariance controller that turns harmonic edge information into bounded forces on the agents and admits distributed evaluation. Under the stated assumptions, the regularized closed loop has a unique global trajectory from every safe initial state. For this controlled dynamics, the initial state and model determine an offline sequence of rational numbers $c_T$. Each stage is a finite calculation and does not query trajectory data. The sequence is non-decreasing and converges to the cumulative normalized hole area. When the cost is finite, the average hole area on every fixed-length
window tends to zero as the window moves forward. Yet this average statement implies neither pointwise convergence to zero nor finite-time hole closure. The numerical demo in Section~\ref{sec:benchmark} illustrates the controller architecture, with both initial holes closing in finite time.

Despite the results in this paper, a key difficulty persists: construction of $q_T$ and quantifier elimination in Section~\ref{sec:finite-area-optimization} are still heavy in computation, prohibiting an efficient algorithm that tests whether the holes of a given initial state will close. Apart from addressing this problem, the next steps include weakening assumptions (A2), (A3), and (A4) and testing an actual distributed implementation rather
than the current synchronous emulation.

\section*{Acknowledgments}
This work was supported by the Air Force Office of Scientific Research (AFOSR) under award number FA9550-23-1-0046.

In writing this paper, large language models (LLMs), specifically, the Codex models~\cite{openai2026codex}, were used for polishing human-written text, and for enhancing logical continuity and consistency. LLMs were used for verifying the correctness of the mathematical proofs and for increasing clarity of the writing. LLMs were also used in implementation of the numerical demos, including generation of parts of the code for simulation and visualization.
All edits suggested by AI were manually reviewed by the authors before implementation. All ideas and technical details originated purely from the authors.

\bibliographystyle{IEEEtran}
\bibliography{refs}

@article{reynolds1987flocks,
  author  = {Reynolds, Craig W.},
  title   = {Flocks, Herds and Schools: A Distributed Behavioral Model},
  journal = {ACM SIGGRAPH Computer Graphics},
  volume  = {21},
  number  = {4},
  pages   = {25--34},
  year    = {1987}
}

@article{vicsek1995novel,
  author  = {Vicsek, Tam{\'a}s and Czir{\'o}k, Andr{\'a}s and Ben-Jacob, Eshel and Cohen, Inon and Shochet, Ofer},
  title   = {Novel Type of Phase Transition in a System of Self-Driven Particles},
  journal = {Physical Review Letters},
  volume  = {75},
  number  = {6},
  pages   = {1226--1229},
  year    = {1995}
}

@article{jadbabaie2003coordination,
  author  = {Jadbabaie, Ali and Lin, Jie and Morse, A. Stephen},
  title   = {Coordination of Groups of Mobile Autonomous Agents Using Nearest Neighbor Rules},
  journal = {IEEE Transactions on Automatic Control},
  volume  = {48},
  number  = {6},
  pages   = {988--1001},
  year    = {2003}
}

@article{olfatisaber2006flocking,
  author  = {Olfati-Saber, Reza},
  title   = {Flocking for Multi-Agent Dynamic Systems: Algorithms and Theory},
  journal = {IEEE Transactions on Automatic Control},
  volume  = {51},
  number  = {3},
  pages   = {401--420},
  year    = {2006}
}

@article{cucker2007emergent,
  author  = {Cucker, Felipe and Smale, Steve},
  title   = {Emergent Behavior in Flocks},
  journal = {IEEE Transactions on Automatic Control},
  volume  = {52},
  number  = {5},
  pages   = {852--862},
  year    = {2007}
}

@article{desilva2007coverage,
  author  = {de Silva, Vin and Ghrist, Robert},
  title   = {Coverage in Sensor Networks via Persistent Homology},
  journal = {Algebraic \& Geometric Topology},
  volume  = {7},
  pages   = {339--358},
  year    = {2007}
}

@article{ghrist2008barcodes,
  author  = {Ghrist, Robert},
  title   = {Barcodes: The Persistent Topology of Data},
  journal = {Bulletin of the American Mathematical Society},
  volume  = {45},
  number  = {1},
  pages   = {61--75},
  year    = {2008}
}

@book{edelsbrunner2010computational,
  author    = {Edelsbrunner, Herbert and Harer, John},
  title     = {Computational Topology: An Introduction},
  publisher = {American Mathematical Society},
  year      = {2010}
}

@article{tahbazsalehi2010coverage,
  author  = {Tahbaz-Salehi, Alireza and Jadbabaie, Ali},
  title   = {Distributed Coverage Verification in Sensor Networks Without Location Information},
  journal = {IEEE Transactions on Automatic Control},
  volume  = {55},
  number  = {8},
  pages   = {1837--1849},
  year    = {2010},
  doi     = {10.1109/TAC.2010.2047541}
}

@misc{yadokoro2023weighted,
  author       = {Yadokoro, Shunsaku and Bhattacharya, Subhrajit},
  title        = {Weighted Combinatorial {Laplacian} and its Application to Coverage Repair in Sensor Networks},
  howpublished = {arXiv:2312.04825},
  year         = {2023}
}

@inproceedings{elamvazhuthi2016coverage,
  author    = {Elamvazhuthi, Karthik and Adams, Chase and Berman, Spring},
  title     = {Coverage and Field Estimation on Bounded Domains by Diffusive Swarms},
  booktitle = {Proc. IEEE Conference on Decision and Control (CDC)},
  pages     = {2867--2874},
  year      = {2016},
  doi       = {10.1109/CDC.2016.7798696}
}

@article{rysz2025simplicial,
  author  = {Rysz, Maciej and Pardalos, Panos M. and Mehta, Siddhartha S.},
  title   = {A Simplicial Homology Approach for Assessing and Rectifying Coverage of Sensor Networks for Improved Crop Management},
  journal = {European Journal of Operational Research},
  volume  = {325},
  number  = {1},
  pages   = {204--218},
  year    = {2025},
  doi     = {10.1016/j.ejor.2025.03.010}
}

@inproceedings{derenick2010simplicial,
  author    = {Derenick, Jason and Kumar, Vijay and Jadbabaie, Ali},
  title     = {Towards Simplicial Coverage Repair for Mobile Robot Teams},
  booktitle = {Proc. IEEE International Conference on Robotics and Automation (ICRA)},
  pages     = {5472--5477},
  year      = {2010},
  doi       = {10.1109/ROBOT.2010.5509808}
}

@inproceedings{muhammad2007dynamic,
  author    = {Muhammad, A. and Jadbabaie, A.},
  title     = {Dynamic Coverage Verification in Mobile Sensor Networks via Switched Higher Order {Laplacians}},
  booktitle = {Proc. Robotics: Science and Systems III (RSS)},
  year      = {2007},
  doi       = {10.15607/RSS.2007.III.039}
}

@misc{feng2026poriferous,
  author       = {Feng, Xun and Zhai, Chao},
  title        = {Distributed Coverage Control on Poriferous Surface via Poly-Annulus Conformal Mapping},
  howpublished = {arXiv:2601.13688},
  year         = {2026}
}

@article{yang2022simplicial,
  author  = {Yang, Maosheng and Isufi, Elvin and Schaub, Michael T. and Leus, Geert},
  title   = {Simplicial Convolutional Filters},
  journal = {IEEE Transactions on Signal Processing},
  volume  = {70},
  pages   = {4633--4648},
  year    = {2022},
  doi     = {10.1109/TSP.2022.3207045}
}

@inproceedings{shuman2011chebyshev,
  author    = {Shuman, David I. and Vandergheynst, Pierre and Frossard, Pascal},
  title     = {Chebyshev Polynomial Approximation for Distributed Signal Processing},
  booktitle = {2011 International Conference on Distributed Computing in Sensor Systems and Workshops},
  pages     = {1--8},
  year      = {2011},
  doi       = {10.1109/DCOSS.2011.5982158},
  eprint    = {1105.1891},
  archivePrefix = {arXiv}
}

@article{zhang2020barriers,
  author  = {Zhang, Leiming and Sadler, Brian M. and Blum, Rick S. and Bhattacharya, Subhrajit},
  title   = {Inter-cluster Transmission Control Using Graph Modal Barriers},
  journal = {IEEE Transactions on Signal and Information Processing over Networks},
  volume  = {7},
  pages   = {275--293},
  year    = {2021},
  doi     = {10.1109/TSIPN.2021.3071219}
}

@misc{nistDLMF10352,
  author       = {{NIST Digital Library of Mathematical Functions}},
  title        = {Jacobi--Anger Expansions, Equation 10.35.2},
  howpublished = {\url{https://dlmf.nist.gov/10.35.E2}}
}

@article{schaub2020random,
  author  = {Schaub, Michael T. and Benson, Austin R. and Horn, Paul and Lippner, Gabor and Jadbabaie, Ali},
  title   = {Random Walks on Simplicial Complexes and the Normalized {Hodge} 1-{Laplacian}},
  journal = {SIAM Review},
  volume  = {62},
  number  = {2},
  pages   = {353--391},
  year    = {2020}
}

@article{hansen2021opinion,
  author  = {Hansen, Jakob and Ghrist, Robert},
  title   = {Opinion Dynamics on Discourse Sheaves},
  journal = {SIAM Journal on Applied Mathematics},
  volume  = {81},
  number  = {5},
  pages   = {2033--2060},
  year    = {2021},
  doi     = {10.1137/20M1341088}
}

@inproceedings{hanks2025cellular,
  author    = {Hanks, Tyler and Riess, Hans and Cohen, Samuel and Gross, Trevor and Hale, Matthew and Fairbanks, James},
  title     = {Distributed Multi-Agent Coordination over Cellular Sheaves},
  booktitle = {Proc. IEEE Conference on Decision and Control (CDC)},
  pages     = {2461--2468},
  year      = {2025},
  doi       = {10.1109/CDC57313.2025.11312066}
}

@article{lasserre2008nonlinear,
  author  = {Lasserre, Jean B. and Henrion, Didier and Prieur, Christophe and Tr{\'e}lat, Emmanuel},
  title   = {Nonlinear Optimal Control via Occupation Measures and {LMI}-Relaxations},
  journal = {SIAM Journal on Control and Optimization},
  volume  = {47},
  number  = {4},
  pages   = {1643--1666},
  year    = {2008}
}

@article{henrion2024occupation,
  author  = {Henrion, Didier and Korda, Milan and Kru{\v{z}}{\'i}k, Martin and Rios-Zertuche, Rodolfo},
  title   = {Occupation Measure Relaxations in Variational Problems: The Role of Convexity},
  journal = {SIAM Journal on Optimization},
  volume  = {34},
  number  = {2},
  pages   = {1708--1731},
  year    = {2024},
  doi     = {10.1137/23M1557088}
}

@book{basu2006algorithms,
  author    = {Basu, Saugata and Pollack, Richard and Roy, Marie-Fran{\c{c}}oise},
  title     = {Algorithms in Real Algebraic Geometry},
  edition   = {2nd},
  series    = {Algorithms and Computation in Mathematics},
  volume    = {10},
  publisher = {Springer},
  address   = {Berlin, Heidelberg},
  year      = {2006},
  doi       = {10.1007/3-540-33099-2}
}

@misc{openai2026codex,
  author       = {{OpenAI}},
  title        = {{Codex}},
  howpublished = {AI coding assistant},
  url          = {https://developers.openai.com/codex/},
}

\appendices
\begingroup
\renewcommand{\baselinestretch}{0.77}\selectfont
\setlength{\parskip}{1pt}
\makeatletter
\renewenvironment{proof}[1][\proofname]{\par
  \pushQED{\qed}%
  \normalfont \topsep\z@\partopsep\z@\relax
  \trivlist
  \item[\hskip\labelsep\itshape #1\@addpunct{.}]\ignorespaces
}{%
  \popQED\endtrivlist\@endpefalse
}
\makeatother
\section{Constructions}
\label{app:constructions}

\theoremstyle{appendixplain}
\newtheorem{auxlemma}{Lemma}
\renewcommand{\theauxlemma}{\Alph{auxlemma}}


\subsection{Construction of $F_k$}
\label{app:field-construction}

Fix $T$ and $k$, and let $d_z=4N$. Choose a rational $0<\gamma_T<r_T-a$. On a rational grid, retain every cube that meets $K_T$ and refine until its diameter is below $\gamma_T/(2\sqrt{2})$. These cubes lie in the safe set, and a consistent subdivision gives a compact rational polyhedral neighborhood $K_T\subset P_T\subset\Omega$ with a finite conforming rational triangulation $\mathcal T_k$. At each distinct vertex $\zeta$ of $\mathcal T_k$, evaluate the known formula
$F(\zeta)$ directly and store one shared rational value
$c_{\zeta,k}\in\mathbb Q^{d_z}$ satisfying
$\lVert c_{\zeta,k}-F(\zeta)\rVert_2\le\eta_k$.
For a simplex $\sigma=\operatorname{conv}\{\zeta_0,\ldots,\zeta_{d_z}\}$,
let $\lambda_{\sigma,j}$ be its rational barycentric coordinates and define
\begin{equation}
F_k(z)=\sum_{j=0}^{d_z}\lambda_{\sigma,j}(z)c_{\zeta_j,k},
\;\; z\in\sigma.
\label{eq:field-interpolant}
\end{equation}

The same model data and positive safety margin give a rational Lipschitz upper
bound $L_T\ge1$ for $F$ on the retained simplices. With $h_k$ denoting the
maximum simplex diameter, refine until
\begin{equation}
h_k\le 2^{-k-1} / L_T,
\;\; \eta_k=2^{-k-1}.
\label{eq:field-mesh-tolerances}
\end{equation}
Finite refinement achieves these bounds. Shared nodal values make $F_k$
continuous across simplex faces. For each simplex $\sigma$ and
$z\in\sigma\cap K_T$, barycentric nonnegativity and
\eqref{eq:field-mesh-tolerances} give
$\lVert F_k(z)-F(z)\rVert_2\le\eta_k+L_Th_k\le2^{-k}$.
Taking the supremum over $K_T$ proves
Proposition~\ref{prop:field-approximation}.

\subsection{Construction of $q_T$}
\phantomsection
\label{par:area-lower-approximation}

We construct $q_T$ satisfying \eqref{eq:area-approximation-properties}. Start with a fixed graph $\mathcal E$ on $[N]$ and draw each edge as the
straight segment between its endpoints. Let $\mathcal W_{\mathcal E}(q)$ be
the set of points about which some polygonal cycle in $\mathcal E$ has
nonzero winding number. The graph segments themselves have zero area. The
normalized area enclosed by the graph but left uncovered by the sensing disks
is
\begin{equation}
h_{\mathcal E}(q)
=|\mathcal W_{\mathcal E}(q)\setminus A(q)|/(8Nr_s)^2.
\label{eq:graph-hole-readout}
\end{equation}
Take two configurations $q,\widetilde q$ of the same labeled nodes, ordered
so that $\lVert\widetilde q\rVert_\infty\le\lVert q\rVert_\infty$.
Set $Q=\lVert q\rVert_\infty$ and
$\delta=\max_i\lVert q_i-\widetilde q_i\rVert$, so both configurations
lie in $[-Q,Q]^{2N}$. Take the translation direction of one disk as
horizontal. At each height, its left and right boundary points each move by
at most $\delta$. These give an area bound
$(2r_s)(2\delta)=4r_s\delta$. Moving all $N$ disks therefore changes
coverage over at most $4r_sN\delta$ of area.

Next set $q(s)=(1-s)q+s\widetilde q$ for $0\le s\le1$. The graph
$\mathcal E$ is fixed throughout the position box, including configurations
where its edges exceed the sensing range $2r_s$. Each edge is bounded by the
diagonal of $[-Q,Q]^2$, namely $2\sqrt{2}Q\le3Q$, and every point on it has
speed at most $\delta$. Thus each edge sweeps area at most
$\int_0^1 3Q\delta\,\mathrm ds=3Q\delta$. Cycle winding numbers are
unchanged outside the swept regions, so the graph-enclosed region changes
over at most $3Q|\mathcal E|\delta$ of area. Adding the two area bounds and
dividing by $(8Nr_s)^2$ gives
\begin{equation}
|h_{\mathcal E}(q)-h_{\mathcal E}(\widetilde q)|
\le L_{\mathcal E,Q}\delta,\;\;
L_{\mathcal E,Q}=\frac{4r_sN+3Q|\mathcal E|}{(8Nr_s)^2}.
\label{eq:graph-readout-lipschitz}
\end{equation}

It remains to approach the sensing graph without making a discontinuous edge
decision. For every complete-graph edge $e=\{i,j\}$, define
$b_e(q)=4r_s^2-\lVert q_i-q_j\rVert^2$ and set
\begin{equation}
p_{e,T}(q)=\min\left(1,\max(0,Tb_e(q))\right).
\label{eq:edge-probability}
\end{equation}
The number $p_{e,T}$ acts as the inclusion weight of edge $e$; its exclusion
weight is $1-p_{e,T}$. Averaging over the finitely many graphs gives
\begin{equation}
q_T(q)=\sum_{\mathcal E}
\left[\prod_{e\in\mathcal E}p_{e,T}(q)\right]
\left[\prod_{e\notin\mathcal E}(1-p_{e,T}(q))\right]
h_{\mathcal E}(q).
\label{eq:area-lower-envelope}
\end{equation}
For the sensing graph $\mathcal E(q)$, a point outside $A(q)$ is encircled
by a cycle exactly when it lies in a hole, so
$h_{\mathcal E(q)}(q)=h(q)$. Since $h_{\mathcal E}$ is nondecreasing under
edge addition, the convex combination in
\eqref{eq:area-lower-envelope}, supported on sensing subgraphs, gives
$0\le q_T\le h$. Continuity follows from
\eqref{eq:graph-readout-lipschitz} and \eqref{eq:edge-probability}.

For fixed $q$, since $p_{e,T}(q)$ is non-decreasing with $T$, $q_T$ is also non-decreasing. Every nonsensing edge has zero weight. Once $Tb_e(q)\ge1$ for every sensing edge, \eqref{eq:area-lower-envelope} reduces to $h_{\mathcal E(q)}(q)=h(q)$. Hence $q_T(q)\uparrow h(q)$, proving \eqref{eq:area-approximation-properties}.

\subsection{Construction of $\widehat q_{T,m}$}
\label{app:objective-construction}

Fix $T$ and $m$. Choose a rational $Q_+>0$ so that $P=[-Q_+,Q_+]^{2N}$ contains all $q$-components of $K_T$, and take a rational Lipschitz upper bound $\Lambda_T\ge1$ for $q_T$ on $P$; its
construction is given below. Take a  rational triangulation
such that any two configurations $q,\widetilde q$ in the same simplex satisfy $\max_i\|q_i-\widetilde q_i\|_2\le\rho_m \coloneqq 2^{-m-1}/\Lambda_T.$

At each shared vertex $q^\nu$, compute $y_\nu\in\mathbb Q$ satisfying
$|y_\nu-q_T(q^\nu)|\le2^{-m-1}$. 
For each graph, divide a box containing the region in
\eqref{eq:graph-hole-readout} into small squares.
Discard squares verified to lie entirely outside the region.
The total area of squares verified to lie entirely inside is a
lower bound; adding the area of every square that intersects the boundary gives an upper bound. Normalize these bounds and evaluate \eqref{eq:area-lower-envelope} using rational input bounds. Refine the grid until the resulting gap between upper and lower bounds falls below $2^{-m}$, and take its midpoint as $y_\nu$. For $q\in P$, let $\lambda_\nu(q)\ge0$ be its barycentric coordinate with respect to the vertices $q^\nu$ of the containing simplex and define the piecewise affine approximation
\[
\widehat q_{T,m}(q)=\sum_\nu\lambda_\nu(q)y_\nu.
\]
Then it follows that
\[
|\widehat q_{T,m}(q)-q_T(q)|
\le\sum_\nu\lambda_\nu(q)
\left(2^{-m-1}+\Lambda_T\rho_m\right)=2^{-m}.
\]
Thus this construction satisfies \eqref{eq:objective-approximation}. To choose $\Lambda_T$, note that every graph-enclosed region lies in
$[-Q_+,Q_+]^2$, so $0\le h_{\mathcal E}\le H$ for a rational
$H\ge4Q_+^2/(8Nr_s)^2$. For $q,\widetilde q\in P$, set
$\delta=\max_i\lVert q_i-\widetilde q_i\rVert$. The edge lengths are at
most $3Q_+$, and
\[
\begin{aligned}
|p_{e,T}(q)&-p_{e,T}(\widetilde q)|
\le T\left|\lVert q_i-q_j\rVert^2
-\lVert\widetilde q_i-\widetilde q_j\rVert^2\right|\\
&\le T(3Q_++3Q_+)(\delta+\delta)=12TQ_+\delta.
\end{aligned}
\]
\eqref{eq:area-lower-envelope} then gives
\[
\begin{aligned}
&|q_T(q)-q_T(\widetilde q)|
\\
&\le\max_{\mathcal E}|h_{\mathcal E}(q)-h_{\mathcal E}(\widetilde q)|+H\sum_e|p_{e,T}(q)-p_{e,T}(\widetilde q)|\\
&\le\left(\max_{\mathcal E}L_{\mathcal E,Q_+}
+12TQ_+H\binom{N}{2}\right)\delta.
\end{aligned}
\]
Any rational $\Lambda_T\ge1$ not smaller than this coefficient suffices.

\section{Proofs}
\label{app:proofs}

\begin{proof}[Proof of Lemma~\ref{lem:properties}]
Summing the pair forces gives
\begin{align*}
\sum_i(q_i-\overline{q})^{\top}F_i^{H}
&=\lambda\sum_e\left((q_i-\overline{q})^{\top}F_i^{e}
+(q_j-\overline{q})^{\top}F_j^{e}\right)\\
&=-\lambda\kappa\sum_e a_e\,s_e\,\|q_j-q_i\|^2\le0.
\end{align*}
The last inequality uses $\lambda,\kappa,a_e,s_e\ge0$.
\end{proof}

\begin{auxlemma}
\label{lem:occsingleton}
Fix $T\ge1$ and use $K_T$ and $X_T$ defined in Section~\ref{sec:initial-state-construction}.  Among the finite
nonnegative measures on $X_T$, the equations
\[
\int_{X_T}\left(\partial_t\psi+F^{\top}\nabla_z\psi\right)\mathrm d\mu
=-\psi(0,z_0)
\]
\hspace*{1pt}for every $C^1$ function $\psi$ with $\psi(T,\cdot)=0$ have the
single solution $\mu_T=\mathrm dt\otimes\delta_{z(t)}$ with $\mu_T(X_T)=T$.
\end{auxlemma}

\begin{proof}[Proof of Lemma~\ref{lem:occsingleton}]
By definition, for $g\in C(X_T)$, one has that $\int_{X_T}g\mathrm d\mu_T=\int_0^T g(t,z)\mathrm dt$. Taking $g=\partial_t\psi+F^\top\nabla_z\psi$ and applying
\eqref{eq:occupation-liouville} gives the stated equations. Conversely, for any solution $\mu$, taking $\psi(t,z)=\int_t^T\zeta(s)\,\mathrm ds$ gives $ \int_{X_T}\zeta(t)\,\mathrm d\mu =\int_0^T\zeta(t)\,\mathrm dt$, with $\zeta\in C([0,T])$. Hence $\pi_t\mu=\mathrm dt$ and $\mu(X_T)=T$. Meanwhile, local Lipschitzness of $F$ gives a constant $L$ such that $\lVert F(x)-F(y)\rVert\le L\lVert x-y\rVert$ for $x,y\in K_T$. Let $z(t)$ be the trajectory under $F$ and take $\Psi(t,x)=(T-t)e^{-2Lt}\left\lVert x-z(t)\right\rVert^2$. Then $\Psi\in C^1$ and $\Psi(T,\cdot) = 0$. Since $\dot z=F(z)$, Cauchy--Schwarz and the Lipschitz bound give
\[ \partial_t\Psi+F(x)^{\top}\nabla_x\Psi
\le-e^{-2Lt}\left\lVert x-z(t)\right\rVert^2.\]
The integral identity for $\Psi$ has zero right-hand side.  Its nonpositive integrand
must therefore vanish $\mu$-almost everywhere. Hence
$\lVert x-z(t)\rVert=0$ almost everywhere. This
finally gives $\mu=\mathrm dt\otimes\delta_{z(t)}$.
\end{proof}

\begin{auxlemma}
\label{lem:staticeval}
Fix a positive integer $T$ and a bounded continuous function
$g:X_T\to\mathbb R$, with uniform rational
approximations on $X_T$ with any prescribed error bound and corresponding trajectory integral in Lemma~\ref{lem:occsingleton}. For every rational $\varepsilon>0$, without querying the trajectory, a terminating
finite calculation involving the model, the
initial state $z_0$, the compact set $K_T$, and these approximation data returns rational numbers $L\le U$ with
\[
U-L<\varepsilon,\;\;
L\le\int_{X_T}g\,\mathrm d\mu_T\le U.
\]
\end{auxlemma}

\begin{proof}[Proof of Lemma~\ref{lem:staticeval}]

Use the terminal-zero functions $\psi_\ell$ defined in
Section~\ref{sec:initial-state-construction}. They are dense in the required
$C^1$ class. For each $k$, Proposition~\ref{prop:field-approximation} supplies the field
$F_k$ used in $a_{\ell k}$, while $\Gamma_\ell$ and $\beta_{\ell k}$ are chosen as
in Section~\ref{sec:initial-state-construction}.
Equation~\eqref{eq:main-test-residual} supplies the tolerance in the finite
constraints. Let $\mathcal P_{T,m}$ be the set of
nonnegative measures $\mu$ on $X_T$ satisfying $\mu(X_T)=T$ and
\begin{gather*}
\left|\int_{X_T}a_{\ell k}\,\mathrm d\mu+\beta_{\ell k}\right|
\le\left(T\Gamma_\ell+1\right)2^{-k},\;\; 1\le\ell,k\le m.
\end{gather*}
The true measure belongs to every $\mathcal P_{T,m}$ by
\eqref{eq:main-test-residual}. These sets are nested. They are also
weakly compact because $X_T$ is compact, every measure is nonnegative with
the same finite mass $T$, and the defining inequalities are weakly closed.
For each fixed $\ell$, letting $k\to\infty$ recovers its exact integral identity; density
of the $\psi_\ell$ family and Lemma~\ref{lem:occsingleton} therefore give
$\bigcap_{m\ge1}\mathcal P_{T,m}=\{\mu_T\}$.
Thus
\begin{equation}
\sup_{\mu\in\mathcal P_{T,m}}
\left|\int_{X_T}g\,\mathrm d\mu-
\int_{X_T}g\,\mathrm d\mu_T\right|\to 0.
\label{eq:nested-observable-convergence}
\end{equation}
We now make a fixed stage finite. Choose a continuous rational
piecewise-polynomial $g_m$ on $X_T$ with
$\lVert g_m-g\rVert_{\infty,X_T}\le2^{-m}$, and define
$H_m(x)=(\left(a_{\ell k}(x)\right)_{1\le\ell,k\le m},g_m(x))
\in\mathbb R^{m^2+1}$.
For every $\mu\in\mathcal P_{T,m}$, the normalized integral of $H_m$ lies in
$\operatorname{conv}H_m(X_T)$.  Carath\'eodory's theorem in convex geometry therefore gives
points $x_r\in X_T$, coefficients $\lambda_r\ge0$ with
$\sum_r\lambda_r=1$, and at most $m^2+2$ terms such that
\[
\frac{1}{T}\int_{X_T}H_m\,\mathrm d\mu
=\sum_{r=1}^{M}\lambda_rH_m(x_r).
\]
Set $w_r=T\lambda_r$, padding to $M=m^2+2$ with zero weights.
The extrema $j_m^\pm$ of $\sum_rw_rg_m(x_r)$ under
\eqref{eq:finite-static-constraints} equal those of
$\int g_m\,\mathrm d\mu$ over $\mathcal P_{T,m}$.
Since $\mu(X_T)=T$,
$|\int(g_m-g)\,\mathrm d\mu|\le T2^{-m}$.
Together with \eqref{eq:nested-observable-convergence}, this gives
$j_m^\pm\to\int g\,\mathrm d\mu_T$ and
\begin{equation}
j_m^- -T2^{-m}\le\int_{X_T}g\,\mathrm d\mu_T
\le j_m^+ +T2^{-m}.
\label{eq:finite-static-outward-enclosure}
\end{equation}
The procedure in Section~\ref{sec:finite-area-optimization} converts
\eqref{eq:finite-static-outward-enclosure} into rational bounds with vanishing
gap, so refinement terminates.
\end{proof}

\endgroup

\end{NoHyper}
\end{document}